\documentclass{article}

\usepackage{microtype}
\usepackage{graphicx}
\usepackage{subcaption}
\usepackage{booktabs} 

\usepackage[nohyperref,accepted]{icml2018}

\usepackage[table]{xcolor}
\usepackage{color}
\usepackage{tikz}
\usetikzlibrary{automata,arrows,positioning,calc}
\usepackage{tkz-graph}

\usepackage{amsmath, amsthm, amssymb,mathtools}
\usepackage{url}

\usepackage{multirow}

\newtheorem{proposition}{Proposition}

\newcommand{\xdim}{d}
\newcommand{\nsamples}{n}
\newcommand{\nbins}{k}
\newcommand{\distweight}{p}

\newcommand{\wvec}{\mathbf{w}}
\newcommand{\xvec}{\mathbf{x}}

\newcommand{\phivec}{\boldsymbol{\phi}}
\newcommand{\thetavec}{{\boldsymbol{\theta}}}

\newcommand{\RR}{\mathbb{R}}

\newcommand{\myparagraph}[1]{\vspace{0.01cm}
\noindent \textbf{#1} \hspace{0.2cm}}
\newcommand*{\argmin}{\mathop{\mathrm{argmin}}}

\newcommand{\inlinevec}[2]{%
  \ensuremath{\Bigl[\negthinspace\begin{smallmatrix}#1\\#2\end{smallmatrix}\Bigr]}}

\definecolor{beaublue}{rgb}{0.74, 0.83, 0.9}
\definecolor{aliceblue}{rgb}{0.94, 0.97, 1.0}

\icmltitlerunning{Improving Regression Performance with Distributional Losses} 

\begin{document}

\twocolumn[
\icmltitle{Improving Regression Performance with Distributional Losses}

\begin{icmlauthorlist}
\icmlauthor{Ehsan Imani}{uofa}
\icmlauthor{Martha White}{uofa}
\end{icmlauthorlist}
\icmlaffiliation{uofa}{Department of Computing Science, University of Alberta, Edmonton}
\icmlcorrespondingauthor{Martha White}{whitem@ualberta.ca}

\icmlkeywords{regression,KL-Divergence}

\vskip 0.2in
]

\printAffiliationsAndNotice{}

\begin{abstract} 
There is growing evidence that converting targets to soft targets in supervised learning can provide considerable gains in performance. Much of this work has considered classification, converting hard zero-one values to soft labels---such as by adding label noise, incorporating label ambiguity or using distillation. In parallel, there is some evidence from a regression setting in reinforcement learning that learning distributions can improve performance. In this work, we investigate the reasons for this improvement, in a regression setting. We introduce a novel distributional regression loss, and similarly find it significantly improves prediction accuracy. We investigate several common hypotheses, around reducing overfitting and improved representations. We instead find evidence for an alternative hypothesis: this loss is easier to optimize, with better behaved gradients, resulting in improved generalization. We provide theoretical support for this alternative hypothesis, by characterizing the norm of the gradients of this loss. 
\end{abstract} 
%
\section{Introduction}

The choice of problem formulation for regression has a large impact on prediction performance on new data---generalization performance. 
There is an extensive literature on problem formulations to promote generalization, including robust losses \citep{huber2011robust,ghosh2017robust,barron2017amore};
proxy losses and reductions between problems \citep{langford2006predicting};
the addition of regularization to impose constraints or preferences on the solution;
 the addition of label noise \citep{szegedy2016rethinking};
 and even ensuring multiple tasks are learned simultaneously, rather than separately, as in multi-task learning \citep{caruana1998multitask}. There is typically a goal in mind---such as classification accuracy or absolute error for regression---but those losses are not necessarily directly minimized. 

In recent years, there has been a particular focus on learning representations with neural networks that generalize better. 
With fixed representations, the loss or problem formulation can only have so much impact, because the learned function is a linear function of inputs. With (deep) neural networks, however, the performance can vary widely, based even on simple modifications such as the initialization \citep{glorot2010understanding}. 
Particularly in classification, modifying the outputs can significantly improve performance.
An extensive empirical study on classification and age prediction \citep{gao2017deep}, under label ambiguity, showed that data augmentation on the label side---putting a distribution over an ambiguous label---significantly improved test accuracy, validated also by other work on age estimation~\citep{rothe2018deep}.
Work on model compression \citep{ba2013do,urban2016do} and distillation \citep{hinton2015distilling} highlight that a smaller student model can be trained to capture the generalization ability of a larger teacher model. In general, there is a growing literature on data augmentation and label smoothing, that advocates for reduced overfitting and improved generalization from modifying the outputs \citep{norouzi2016reward,szegedy2016rethinking,xie2016disturblabel,miyato2016distributional,pereyra2017regularizing} 
and in reinforcement learning where learning distributional outputs, rather than means, improves performance \citep{bellemare2017distributional}.

There has been some work---though considerably less---towards understanding the impact of the properties of the loss that promote effective optimization. 
There is a recent insight that minimizing training time increases generalization performance \citep{hardt2015train}, motivating the design of losses that can be more easily optimized.
Though not the focus in data augmentation, there have been some insights about loss properties. 
\citet{gao2017deep} showed that their data augmentation approach provided a faster convergence rate (see their Figure 8).
\citet{pereyra2017regularizing} showed that label smoothing and their regularizer penalizing confident predictions for classification provided smoother gradient norms than without regularization. 
\citet{bellemare2017distributional} hypothesized that the properties of the KL-divergence could have improved learning performance, in a reinforcement learning setting. 
These papers hint at something deeper occurring with the loss, and motivate investigation into not just the conversion of the problem but into the loss itself. 

In this work, we show that the properties of the loss have a significant effect, and better explain the resulting increase in performance than preventing overfitting. We first propose a new loss for regression, called a Histogram Loss (HL). The targets are converted to a \emph{target distribution}, and the KL-divergence taken between a histogram density and this target distribution. The choice of histogram density provides a relatively flexible \emph{prediction distribution}, that nonetheless enables the KL-divergence to be computed efficiently. The prediction is then the expected value of this histogram density.  
This modification could be seen as converting the problem to a more difficult (multi-task) problem---from one output, to multiple values to represent the distribution---that promotes generalization in the learner and reduces overfitting. We show that instead of this hypothesis,
the (optimization) properties of the HL seem to be the key factor in the resulting improved accuracy. We provide a series of empirical results to support this hypothesis. We also characterize the norm of the gradient of the HL
which directly relates to sample complexity \citep{hardt2015train}. The bounds on the variability of the gradient help explain the positive empirical performance of the HL, and further motivate the use of this loss as an alternative for the standard loss for regression. 

\section{Distributional Losses for Regression}

In this section, we introduce the Histogram Loss (HL), which generalizes beyond special cases of soft-target losses used in recent work \citep{norouzi2016reward,szegedy2016rethinking,gao2017deep}. We first introduce the loss and how it can be used for regression. We then relate it to other objectives, including maximum likelihood for regression and other methods that learn distributions.

\subsection{Learning means and distributions}

In regression, it is common to use the squared-error loss, or $\ell_2$ loss. This corresponds to assuming that the continuous target variable $Y$ is Gaussian distributed, conditioned on inputs $\xvec \in \RR^{\xdim}$: $Y \sim \mathcal{N}(\mu = f(\xvec), \sigma^2)$ for a fixed variance $\sigma^2 > 0$ and some function $f: \RR^\xdim \rightarrow \RR$ on the inputs, such as a linear function $f(\xvec) = \langle \xvec, \wvec \rangle$ for weights $\wvec \in \RR^{\xdim}$. The maximum likelihood function $f$ for $\nsamples$ samples $\{\xvec_i, y_i\}$, corresponds to minimizing the $\ell_2$ loss
\vspace{-0.2cm}
\begin{equation}
\min_{f \in \mathcal{F}} \sum_{j=1}^\nsamples (f(\xvec_j) - y_j)^2
\end{equation}
with prediction $f(\xvec) \approx \mathbb{E}[Y | \xvec]$. 

Alternatively, one could consider learning a distribution over $Y$ directly, and then taking the mean of that distribution---or other statistics---to provide a prediction. This additional difficulty seems hardly worth the effort, considering only the mean is required for prediction. However, as motivated above, the increased difficulty could beneficially prevent overfitting and promote generalization. 

There are many options for learning conditional distributions, $p(y | \xvec)$ even when only considering those that use neural networks \citep{bishop1994mixture,tang2013learning, rothe2015dex,bellemare2017distributional}.
The goal of this work, however, is not to provide another method to learn distributions. Rather, the goal is to benefit from inducing a distribution over $Y$,
even if that distribution will subsequently not be used, other than for computing a mean prediction. In our experiments, we will compare to an approach that learns distributions, but only to evaluate regression performance. 
%

\subsection{The Histogram Loss}

Consider predicting a continuous target $Y$ with event space $\mathcal{Y}$, given inputs $\xvec$. Instead of directly predicting $Y$, we select a \emph{target distribution} on $Y | \xvec$. This target distribution is selected upfront, by us, rather than being learned.  
Suppose the target distribution has support $[a,b]$, pdf $p$, and cdf $F$. We would like to learn a parameterized \emph{prediction distribution} $q_\xvec: \mathcal{Y} \rightarrow [0,1]$, conditioned on $\xvec$, by minimizing a KL-divergence to $p$. For any $p$, however, this may be expensive. Further, depending on the parameterization of the prediction distribution, this may also be potentially non-convex in those parameters. 

We propose to restrict the prediction distribution $q_\xvec$ to be a \emph{histogram density}. 
%
Assume $[a,b]$ has been uniformly partitioned into $\nbins$ bins, of width $w_i$, and let function $f: \mathcal{X} \rightarrow [0,1]^\nbins$ provide $\nbins$-dimensional vector $f(\xvec)$ of the coefficients indicating the probability the target is in that bin, given $\xvec$. The density $q_\xvec$ corresponds to a (normalized) histogram, and has density values $f_i(\xvec)/w_i$ per bin. The KL-divergence between $p$ and $q_{\xvec}$ is
\vspace{-0.2cm}
\begin{equation*}
D_{KL}(p||q_{\xvec}) = h(p,q_{\xvec}) - h(p).
\end{equation*}
The second term is the differential entropy---the extension of entropy to continuous random variables. Because the second term only depends on $p$, the aim is to minimize the first term: the cross-entropy between $p$ and $q_{\xvec}$. 
This loss simplifies, due to the form on $q_{\xvec}$:
\begin{align*}
h(p,q_{\xvec}) &= - \int_{a}^{b} p(y) \log q_{\xvec}(y) dy \\
&= - \sum_{i=1}^\nbins \int_{l_i}^{l_i + w_i} p(y) \log \frac{f_i(\xvec)}{w_i} dy \\
&= - \sum_{i=1}^\nbins \log \frac{f_i(\xvec)}{w_i} \underbrace{(F(l_i + w_i) - F(l_i))}_{\distweight_i}.
\end{align*}
In the minimization, the width itself can be ignored, because $\log \tfrac{f_i(\xvec)}{w_i} = \log f_i(\xvec) - \log w_i$, giving the Histogram Loss 
\begin{equation}
HL(p,q_{\xvec}) = - \sum_{i=1}^\nbins \distweight_i \log f_i(\xvec) \label{eq_hl}
.
\end{equation}
This loss has several useful properties. One important property is that it is convex in $f_i(\xvec)$; even if the loss is not convex in all network parameters, it is at least convex on the last layer. The other three benefits are due to restricting the form of the predicted distribution $q_{\xvec}$ to be a histogram density. 
First, the divergence to the full distribution $p$ can be efficiently computed. This contrasts previous work, which samples the KL for a subset of $y$ values \citep{norouzi2016reward,szegedy2016rethinking}. Second, the choice of $p$ is flexible, as long as its CDF can be evaluated for each bin. The weighting $\distweight_i = F(l_i + w_i) - F(l_i)$ can be computed offline once for each sample, making it inexpensive to query repeatedly for each sample during training.  Third, different distributional choices simply result in different weightings in the cross-entropy. This simplicity facilitates interpreting the impact of changing the distributional assumptions on $Y$. 

\subsection{Target distributions and related objectives}

Below, we consider some special cases for $p$ that are of interest and highlight connections to previous work. 

\myparagraph{Truncated Gaussian on $Y | \xvec$ and HL-Gaussian.}
Consider a truncated Gaussian distribution, on support $[a,b]$, as the target distribution. The mean $\mu$ for this Gaussian is the datapoint $y_j$ itself, with fixed variance $\sigma^2$. The pdf $p$ is
%
\begin{equation*}
p(y) = \frac{1}{Z\sigma \sqrt{2 \pi}} e^{-\frac{(y-\mu)^2}{2\sigma^2}}
\end{equation*}
where $Z = \frac{1}{2} (\text{erf}\left(\frac{b-\mu}{\sqrt{2}\sigma}\right) - \text{erf}\left(\frac{a-\mu}{\sqrt{2}\sigma}\right))$,
and the HL has
\small
\begin{equation*}
\distweight_i = \tfrac{1}{2Z} \left(\text{erf}\left(\frac{l_i + w_i - \mu}{\sqrt{2}\sigma}\right) - \text{erf}\left(\frac{l_i - \mu}{\sqrt{2}\sigma}\right) \right)
.
\end{equation*}
\normalsize
This distribution enables significant smoothing over $Y$, through the variance parameter $\sigma^2$. We call this loss HL-Gaussian, defined by number of bins $\nbins$ and variance $\sigma^2$. Based on positive empirical performance, it will be the main HL loss that we advocate for and analyze. 

\myparagraph{Soft Targets and a Histogram Density on $Y| \xvec$.}
In classification, such as multinomial logistic regression, it is typical to assume $Y | \xvec$ is a categorical distribution, where $Y$ is discrete.
The goal is still to estimate $\mathbb{E}[Y | \xvec]$ and when training, hard 0-1 values for $Y$ are used in the cross-entropy. Soft labels, instead of 0-1 labels, can be used by adding label noise \citep{norouzi2016reward,szegedy2016rethinking,pereyra2017regularizing}. This can be seen as an instance of HL, but for discrete $Y$, where a categorical distribution is selected for the target distribution. Minimizing the cross-entropy to these soft-labels corresponds to trying to match such a smoothed target distribution, rather than the original 0-1 categorical distribution. 


Such soft targets have also been considered for ordinal regression, again motivated as label smoothing, for age prediction \citep{gao2017deep,rothe2018deep}. The outputs are smoothed using radial basis function similarities to a set of bin centers. This procedure can be seen as selecting a histogram density for the target distribution, where the coefficients for each bin are determined by these radial basis function similarities. 
The resulting loss is similar to HL-Gaussian, with slightly different $\distweight_i$, though introduced as data augmentation to smooth (ordinal) targets.


\myparagraph{Dirac delta on $Y | \xvec$.}
Finally, we consider the relationship to maximum likelihood. 
For classification, \citet{norouzi2016reward} and \citet{szegedy2016rethinking} used a combination of maximum likelihood and a KL-divergence to a (uniform) distribution. \citet{szegedy2016rethinking} add uniform noise to the labels and \citet{norouzi2016reward} sample from an exponentiated reward distribution, with a temperature parameter, for structured prediction. Both consider only a finite set for $Y$, because they both address classification problems.

The relationship between KL-Divergence and maximum likelihood can be extended to continuous $Y$.
The connection is typically in terms of statistical consistency: the maximum likelihood estimator approaches the minimum of the KL-divergence to the true distribution, if the distributions are of the same parametric form \citep[Theorem 9.13]{wasserman2004all}. They can, however, be connected for finite samples with different distributions. Consider Gaussians centered around datapoints $y_j$, with arbitrarily small variances $\tfrac{1}{2}a^2$: 
\begin{equation*}
\delta_{a,j}(y) = \frac{1}{a^2 \sqrt{\pi}}\exp\left(-\tfrac{(y - y_j)^2}{a^2}\right). 
\end{equation*}
Let the target distribution have $p(y) = \delta_{a,j}(y)$ for each sample. Define function $p_{i,j}: [0,\infty) \rightarrow [0,1]$ as 
%
$p_{i,j}(a) = \int_{l_i}^{l_i + w_i} \delta_{a,j}(y) dy$
.
%
%
For each $y_j$, as $a \rightarrow 0$, $p_{i,j}(a) \rightarrow 1$ if $y_j \in [l_i, l_i + w_i]$ and $p_{i,j}(a) \rightarrow 0$ otherwise. 
So, for $i_j$ s.t. $y_j \!\in [l_{i_j}, l_{i_j} \!+\! w_i]$, 
\begin{align*}
\!\!\lim_{a \rightarrow 0} \! HL(\delta_{a,j},q_{\xvec_j}) &= - \log f_{i_j}(\xvec_j).
\end{align*}
The sum over samples for the HL to the Dirac delta on $Y | \xvec$, then, corresponds to the negative log-likelihood for $q_\xvec$
%
\begin{align*}
\argmin_{f_1, \ldots, f_\nbins} - \sum_{j=1}^\nsamples \log f_{i_j}(\xvec_j)
&= \argmin_{f_1, \ldots, f_\nbins} - \sum_{j=1}^\nsamples \log q_{\xvec_j}(y_j)
.
\end{align*}
Such a delta distribution on $Y | \xvec$ results in one coefficient $\distweight_i$ being 1, reflecting the distributional assumption that $Y$ is certainly in a bin.
In the experiments, we compare to this loss, which we call \textbf{HL-OneBin}.

Using a similar analysis to above, $p(y)$ can be considered as a mixture between $\delta_{a,j}(y)$ and a uniform distribution. For a weighting of $\epsilon$ on the uniform distribution, the resulting loss \textbf{HL-Uniform} has $p_i = \epsilon$ for $i \neq i_j$, and $p_{i_j} = 1- \nbins \epsilon$.


\section{Optimization properties of the HL}

There are at least two motivations for this loss, in terms of promoting the search for effective solutions. 
The first is the stability of gradients, promoting stable gradient descent. The second is a connection to learning optimal policies in reinforcement learning. Both provide some insight that the properties of the HL, during optimization, promote better generalization performance.

\myparagraph{Stable gradients for HL.} 
\citet{hardt2015train} have shown that the generalization performance for stochastic gradient descent is bounded by the number of steps
that stochastic gradient descent takes during training, even for non-convex losses. 
The bound is also dependent on the properties of the loss.
In particular, it is beneficial to have a loss function with small Lipschitz constant $L$, which bounds the norm of the gradient.
Below, we discuss how the HL with a Gaussian distribution (HL-Gaussian) in fact promotes an improved bound on this norm, over both the $\ell_2$ loss and the HL with all weight in one bin (HL-OneBin). 

\newcommand{\lip}{l}

In the proposition bounding the HL-Gaussian gradient, we assume 
\begin{equation}
f_i(\xvec) = \tfrac{\exp(\phivec_\thetavec(\xvec)^\top \wvec_i)}{\sum_{j=1}^\nbins \exp(\phivec_\thetavec(\xvec)^\top \wvec_j)} \label{eq_softmax}
\end{equation}
for some function $\phivec_\thetavec: \mathcal{X} \rightarrow \mathcal{R}^k$ parameterized by a vector of parameters $\thetavec$. For example, $\phivec_\thetavec(\xvec)$ could be the last hidden layer in a neural network, with parameters $\thetavec$ for the entire network up to that layer.  
The proposition provides a bound on the gradient norm in terms of the \emph{current network parameters}. Our goal is to understand how the gradients might vary \emph{locally for the parameters}, as opposed to globally bounding the norm and characterizing the Lipschitz constant only in terms of the properties of the function class and loss. 
\begin{proposition}[Local Lipschitz constant for HL-Gaussian]
Assume $\xvec, y$ are fixed, giving fixed coefficients $\distweight_i$ in HL-Gaussian.
Let $f_i(\xvec)$ be as in \eqref{eq_softmax}, defined by the parameters $\wvec = \{\wvec_1, \ldots, \wvec_\nbins\}$ and $\thetavec$, providing the predicted distribution $q_\xvec$. 
Assume for all $i$ that $\wvec_i^\top \phivec_\thetavec(\xvec)$ is locally $\lip$-Lipschitz continuous w.r.t $\thetavec$ 
\begin{equation}
\|  \nabla_\thetavec (\wvec_i^\top \phivec_\thetavec(\xvec)) \| \le \lip \label{eq_lipschitz}
\end{equation}
Then the norm of the gradient for HL-Gaussian, w.r.t. to all the parameters in the network $\{\thetavec, \wvec\}$, is bounded by
\begin{equation}
\!\!\|  \nabla_{\thetavec,\wvec} HL(p, q_{\xvec})  \|
 \le \! \left(\lip \!+\!  \| \phivec_\thetavec(\xvec) \| \right) \! \sum_{i=1}^\nbins | \distweight_i - f_i(\xvec)| \!\! \label{eq_hl_lipschitz}
\end{equation}
\end{proposition}
\begin{proof}
First consider the gradient of the HL, with explicit details on these computations in Appendix \ref{app_calculations}
\begin{align*}
\frac{\partial}{\partial \wvec_i} \sum_{j=1}^\nbins \distweight_j \log f_j(\xvec)
&=  (\distweight_i -f_i(\xvec) ) \phivec_\thetavec(\xvec)
\end{align*}
The norm of the gradient of HL in Equation $\eqref{eq_hl}$, w.r.t. $\wvec$ which is composed of all the weights $\wvec_i \in \RR^\nbins$ is
\begin{align*}
\!\! \!\Big\| \frac{\partial}{\partial \wvec} \sum_{j=1}^\nbins \distweight_j \log f_j(\xvec) \Big\|
&\le  \sum_{i=1}^\nbins \Big\| \frac{\partial}{\partial \wvec_i} \sum_{j=1}^\nbins \distweight_j \log f_j(\xvec) \Big\|\\
&= \sum_{i=1}^\nbins  \left\| (\distweight_i -f_i(\xvec) ) \phivec_\thetavec(\xvec) \right\| \\
&\le  \sum_{i=1}^\nbins | \distweight_i -f_i(\xvec) | \| \phivec_\thetavec(\xvec) \| 
\end{align*}
Similarly, the norm of the gradient w.r.t. $\thetavec$ is
\begin{align*}
\!\!\!\!\Big\| \!\frac{\partial}{\partial \thetavec} \!\!\sum_{j=1}^\nbins \!\distweight_j \log f_j(\xvec)\Big\| \!
&= \Big\| \sum_{i=1}^\nbins (\distweight_i \!-\!f_i(\xvec) \!) \nabla_\thetavec \wvec_i^\top\!\!\phivec_\thetavec(\xvec) \Big\|\\
&\le \sum_{i=1}^\nbins \left\| (\distweight_i \!-\!f_i(\xvec) \!) \nabla_\thetavec \wvec_i^\top\!\!\phivec_\thetavec(\xvec) \right\|\\
&\le \sum_{i=1}^\nbins | \distweight_i -f_i(\xvec) | \lip 
\end{align*}
Together, these bound the norm $\|  \nabla_{\thetavec,\wvec} HL(p, q_{\xvec})  \|$.
\end{proof}

The results by \citet{hardt2015train} suggest it is beneficial for the local Lipschitz constant---or the norm of the gradient---to be small on each step.
HL-Gaussian provides exactly this property. Besides the network architecture---which we are here assuming is chosen outside of our control---the HL-Gaussian gradient norm is proportional to $|\distweight_i -f_i(\xvec)|$. This number is guaranteed to be less than 1, but generally is likely to be even smaller, especially if $f_i(\xvec)$ reasonably accurately predicts $\distweight_i$. Further, the gradients should push the weights to stay within a range specified by $\distweight_i$, rather than preferring to push some to be very small---close to 0---and others to be close to 1. For example, if $f_i(\xvec)$ starts relatively uniform, then the objective does not encourage predictions $f_i(\xvec)$ to get smaller than $\distweight_i$. If $\distweight_i$ are non-negligible, this keeps $f_i(\xvec)$ away from zero and the loss in a smaller range.  

This contrasts both the norm of the gradient for the $\ell_2$ loss and HL-OneBin. For the $\ell_2$ loss, $(f(\xvec) - y) \inlinevec{\nabla_\thetavec \wvec^\top \phivec_\thetavec(\xvec)}{\phivec_\thetavec(\xvec)}$ is the gradient, giving gradient norm bound $(\lip \!+\!  \| \phivec_\thetavec(\xvec) \|) | f(\xvec) - y|$. The constant $| f(\xvec) - y|$, as opposed to $\sum_{i=1}^\nbins |\distweight_i -f_i(\xvec)|$, can be much larger, even if $y$ is normalized between $[0,1]$, and can vary significantly more. 
HL-OneBin, on the other hand, shares the same constant as HL-Gaussian, but suffers from another problem. The Lipschitz constant $\ell$ in Equation \eqref{eq_lipschitz} will likely be larger, because $\distweight_i$ is frequently zero and so pushes $f_i(\xvec)$ towards zero. This results in larger objective values and pushes $\wvec_i^\top \phivec_\thetavec(\xvec)$ to get larger, to enable $f_i(\xvec)$ to get close to 1. 


\myparagraph{Connection to reinforcement learning.} 
The HL can also be motivated through a connection to maximum entropy reinforcement learning. 
In reinforcement learning, an agent iteratively selects actions and transitions between states, to maximize (long-term) reward. 
The agent's goal is to find an optimal policy, in as few interactions as possible. To do so, the agent begins by exploring more, to then enable more efficient convergence to optimal. Supervised learning can be expressed as a reinforcement learning problem \citep{norouzi2016reward}, where action selection conditioned on a state corresponds to making a prediction conditioned on a feature vector. An alternative view to minimizing prediction error is to search for a policy to make accurate predictions. 

One strategy to efficiently find an optimal policy is through a maximum entropy objective. The policy balances between selecting the action it believes to be optimal---make its current best prediction---and acting more randomly---with high-entropy. For continuous action set $\mathcal{Y}$, the goal is to 
minimize the following objective
\vspace{-0.2cm}
\begin{equation}
\int_{\mathcal{X}} p_s(\xvec) \Big[ - \tau h(q_{\xvec}) -  \int_{\mathcal{Y}} q_\xvec(y) r(y,y_i) dy \Big] d\xvec \label{eq_rl}
\end{equation}
where $\tau > 0$; 
$p_s$ is a distribution over states $\xvec$; $q_\xvec$ is the policy or distribution over actions for a given $\xvec$; and $r(y, y_i)$ is the reward function, such as the negative of the objective $r(y,y_i) = - \tfrac{1}{2} ( y - y_i )^2$. Minimizing \eqref{eq_rl} corresponds to minimizing the KL-divergence across $\xvec$ between $q_{\xvec}$ and the exponentiated payoff distribution 
$p(y) = \frac{1}{Z} \exp(r(y,y_i) / \tau)$
%
where $Z = \int \exp(r(y,y_i) / \tau)$, because 
\vspace{-0.2cm}
\begin{align*}
D_{KL}(q_\xvec || p) 
&= -h(q_{\xvec}) - \int q_\xvec(y) \log p(y) dy \\
&= - h(q_{\xvec}) - \tau^{-1} \int q_\xvec(y) r(y,y_i) dy + \log Z.
\end{align*}
%
%
The  connection between the HL and maximum-entropy reinforcement learning is that both are minimizing a divergence to this exponentiated distribution $p$.
The HL, however, is minimizing $D_{KL}(p || q_{\xvec})$ instead of $D_{KL}(q_\xvec || p)$. For example, Gaussian target distribution with variance $\sigma^2$ corresponds to minimizing $D_{KL}(p || q_{\xvec})$ with $r(y,y_i) = - \tfrac{1}{2} ( y - y_i )^2$ and $\tau = \sigma^2$. 
These two KL-divergences are not the same, but a similar argument to \citet{norouzi2016reward} could be extended for continuous $y$, showing $D_{KL}(q_\xvec || p)$ is upper-bounded by $D_{KL}(p || q_{\xvec})$ plus variance terms. The intuition, then, is that minimizing the HL is promoting an efficient search for an optimal (prediction) policy.

\newcolumntype{b}{>{\columncolor{aliceblue}}l}

\begin{table*}[t]
  \centering
  \setlength\tabcolsep{.7pt}
\begin{tabular}{lllllbl}
	\hline
	Method             & Train objective                          & Train MAE                           & Train RMSE                          & Test objective                              & Test MAE                               & Test RMSE                              \\
	\hline
	Linear Reg.  & $ 6738.719$ {\scriptsize $(\pm 10.024)$}                                     & $ 607.277$ {\scriptsize $(\pm 0.706)$} & $ 820.896$ {\scriptsize $(\pm 0.610)$} & $ 6957.086$ {\scriptsize $(\pm 41.419)$}                                      & $ 616.992$ {\scriptsize $(\pm 2.461)$} & $ 834.077$ {\scriptsize $(\pm 2.485)$} \\
	$\ell_2$         & $ 0.002$ {\scriptsize $(\pm 0.001)$}   & $ 15.624$ {\scriptsize $(\pm 3.353)$}  & $ 20.774$ {\scriptsize $(\pm 4.713)$}  & $ 0.001$ {\scriptsize $(\pm 0.000)$}   & $ 19.110$ {\scriptsize $(\pm 3.034)$}  & $ 29.512$ {\scriptsize $(\pm 3.622)$}  \\
	HL-Gaussian & $ 146.521$ {\scriptsize $(\pm 0.045)$} & $ 5.266$ {\scriptsize $(\pm 0.155)$}   & $ 7.097$ {\scriptsize $(\pm 0.169)$}   & $ 147.300$ {\scriptsize $(\pm 0.102)$} & $ \mathbf{8.992}$ {\scriptsize $(\pm 0.235)$}   & $ \mathbf{19.980}$ {\scriptsize $(\pm 2.169)$}\\
	\hline  
	$\ell_1$  & $ 0.152$ {\scriptsize $(\pm 0.002)$}   & $ 12.084$ {\scriptsize $(\pm 0.665)$}  & $ 16.369$ {\scriptsize $(\pm 0.651)$}  & $ 0.161$ {\scriptsize $(\pm 0.006)$}   & $ 16.180$ {\scriptsize $(\pm 0.606)$}  & $ 38.884$ {\scriptsize $(\pm 3.760)$}  \\
	$\ell_2$+Noise  & $ 0.000$ {\scriptsize $(\pm 0.000)$}   & $ 11.398$ {\scriptsize $(\pm 1.108)$}  & $ 15.184$ {\scriptsize $(\pm 1.466)$}  & $ 0.001$ {\scriptsize $(\pm 0.001)$}   & $ 15.233$ {\scriptsize $(\pm 1.038)$}  & $ 31.077$ {\scriptsize $(\pm 7.616)$}  \\
	$\ell_2$+Clipping   & $ 0.000$ {\scriptsize $(\pm 0.000)$}   & $ 11.090$ {\scriptsize $(\pm 0.382)$}  & $ 14.331$ {\scriptsize $(\pm 0.450)$}  & $ 0.001$ {\scriptsize $(\pm 0.000)$}   & $ 14.795$ {\scriptsize $(\pm 0.362)$}  & $ 23.052$ {\scriptsize $(\pm 0.614)$}  \\
\hline
	HL-OneBin   & $ 7.387$ {\scriptsize $(\pm 0.185)$}   & $ 24.623$ {\scriptsize $(\pm 0.055)$}  & $ 33.732$ {\scriptsize $(\pm 3.072)$}  & $ 59.141$ {\scriptsize $(\pm 1.653)$}  & $ 28.001$ {\scriptsize $(\pm 0.322)$}  & $ 63.290$ {\scriptsize $(\pm 5.249)$}  \\
	HL-Uniform       & $ 7.525$ {\scriptsize $(\pm 0.169)$} & $ 24.603$ {\scriptsize $(\pm 0.016)$} & $ 31.257$ {\scriptsize $(\pm 1.600)$} & $ 58.553$ {\scriptsize $(\pm 1.356)$} & $ 28.012$ {\scriptsize $(\pm 0.392)$} & $ 71.088$ {\scriptsize $(\pm 7.304)$} \\
	MDN                & $ -366.062$ {\scriptsize $(\pm 0.225)$} & $ 14.398$ {\scriptsize $(\pm 1.604)$} & $ 22.977$ {\scriptsize $(\pm 3.759)$} & $-365.004$ {\scriptsize $(\pm 0.543)$} &$ 17.950$ {\scriptsize $(\pm 1.514)$} & $ 28.355$ {\scriptsize $(\pm 2.781)$} \\	
	$\ell_2$+Softmax                & $ 0.000$ {\scriptsize $(\pm 0.000)$} & $ 9.183$ {\scriptsize $(\pm 3.784)$} & $ 12.969$ {\scriptsize $(\pm 5.313)$} & $0.001$ {\scriptsize $(\pm 0.000)$} &$ 12.720$ {\scriptsize $(\pm 3.609)$} & $ 22.383$ {\scriptsize $(\pm 4.525)$} \\	
	\hline
\end{tabular}
  \vspace{-0.2cm}
  \caption{Performance on CT Position dataset. All the numbers are multiplied by $10^2$.}
  \label{tab:ctscan}
  \vspace{-0.4cm}
\end{table*}

\vspace{-0.2cm}
\section{Experiments}

In this section, we investigate the utility of the HL-Gaussian for regression, compared to using an $\ell_2$ loss.
We particularly investigate why the modification to this distributional loss improves performance, designing experiments to test if it is due to (a) the utility of learning distributions or smoothed targets, (b) a bias-variance trade-off from bin size or variance in the HL-Gaussian, (c) an improved representation, (d) nonlinearity introduced by the HL and (e) improved optimization properties of the loss.
 

\myparagraph{Datasets and pre-processing.}
All features are transformed to have zero mean and unit variance. We randomly split the data into train and test sets in each run.\\ 
The \textbf{CT Position} dataset is from CT images of patients \citep{graf20112d}, with 385 features 
and the target set to the relative location of the image. \\ 
The \textbf{Song Year} dataset is a subset of The Million Song Dataset \citep{bertin2011million}, with 90 audio features for a song and target corresponding to the release year. \\
The \textbf{Bike Sharing} dataset \citep{fanaee2014event}, about hourly bike rentals for two years, has 16 features and target set to the number of rented bikes. 

Root mean squared error (RMSE) and mean absolute error (MAE) are reported over 5 runs, with standard errors. 
We include both errors and objective values, on train and test, to provide a more complete picture of
the causes of differences between the losses. 
For space, we only include in-depth results on CT Position in the main body. We summarize the overall conclusions on all three datasets below, and include the tables for Song Year and Bike Sharing in Appendix \ref{app_tables} and more dataset information in Appendix \ref{app_datasets}. 

\newcommand{\figwidthfour}{0.33\textwidth}

\myparagraph{Algorithms.}
We compared several regression strategies, distribution learning approaches and several variants of HL. All the approaches---except for Linear Regression---use the same neural network, with differences only in the output layer. 
The architecture for Song Year is 90-45-45-45-45-1 (4 hidden layers of size 45), for Bike Sharing is 16-64-64-64-64 and for CT Position is 385-192-192-192-192-1. All units employ ReLU activation, except the last layer with linear activations. Unless specified otherwise, all networks using HL have 100 bins. Meta-parameters for comparison algorithms are chosen according to best Test MAE. Network architectures were chosen according to best Test MAE for $\ell_2$, with depth and width varied across 7 different values with final choices being neither biggest nor smallest.  

\textbf{Linear Regression} is included as a baseline, using ordinary least squares with the inputs.\\
\textbf{Squared-error} $\ell_2$ is the neural network trained using the $\ell_2$ loss. The targets are normalized to range $[0,1]$, which was needed to improve stability and accuracy.\\ 
\textbf{Absolute-error} $\ell_1$ is the neural network using the $\ell_1$ loss.\\
\textbf{$\mathbf{\ell_2}$+Noise} is the same as $\ell_2$, except Gaussian noise is added to the targets as a form of augmentation. The standard deviation of the noise is selected from $\{10^{-5}, 10^{-4}, 10^{-3}, 10^{-2}, 10^{-1}\}$. \\
\textbf{$\mathbf{\ell_2}$+Clipping} is the same as $\ell_2$, but with gradient norm clipping during training. The threshold for clipping is selected from $\{0.01, 0.1, 1, 10\}$.\\
\textbf{HL-OneBin} is the HL, with Dirac delta target distribution.\\
\textbf{HL-Uniform} is the HL, with a target distribution that mixes between a delta distribution and the uniform distribution, with a weighting of $\epsilon$ on the uniform and $1-\epsilon$ on the delta, where $\epsilon \in \{10^{-5}, 10^{-4}, 10^{-3}, 10^{-2}, 10^{-1}\}$.\\
\textbf{HL-Gaussian} is the HL, with a truncated Gaussian distribution as the target distribution. The variance $\sigma^2$ is set to the radius of the bins. \\
\textbf{MDN} is a Mixture Density Network \cite{bishop1994mixture} that models the target distribution as a mixture of Gaussian distributions. The original model uses an exponential activation to model the standard deviations. However, inspired by \cite{lakshminarayanan2017simple}, we used softplus activation plus a small constant ($10^{-2}$) to avoid numerical instability. We selected the number of components from $\{2, 4, 8, 16 , 32\}$. Predictions are made by taking the mean of the mixture model given by the MDN. \\
\textbf{$\mathbf{\ell_2}$+Softmax} use a softmax-layer with $\ell_2$ loss, $\sum_{i=1}^\nbins (f_i(\xvec_j) c_i - y_j)^2$ for bin centers $c_i$, with otherwise the same settings as HL-Gaussian. 

We used Scikit-learn \citep{scikit-learn} for the implementations of Linear Regression, and Keras \citep{chollet2015keras} for the neural network models.
All neural network models are trained with mini-batch size 256 using the Adam optimizer \citep{kingma2014adam} with a learning rate 1e-3 and the parameters are initialized according to the method suggested by \citet{lecun1998efficient}. Dropout \citep{srivastava2014dropout} with rate $0.05$ is added to the input layer of all neural networks to avoid overfitting. We trained the networks for 1000 epochs on CT Position, 150 epochs on Song Year and 500 epochs on Bike Sharing.

\begin{figure*}[ht!]
 \vspace{-0.3cm}
\centering
 \begin{subfigure}[b]{\figwidthfour}
         \includegraphics[width=\textwidth]{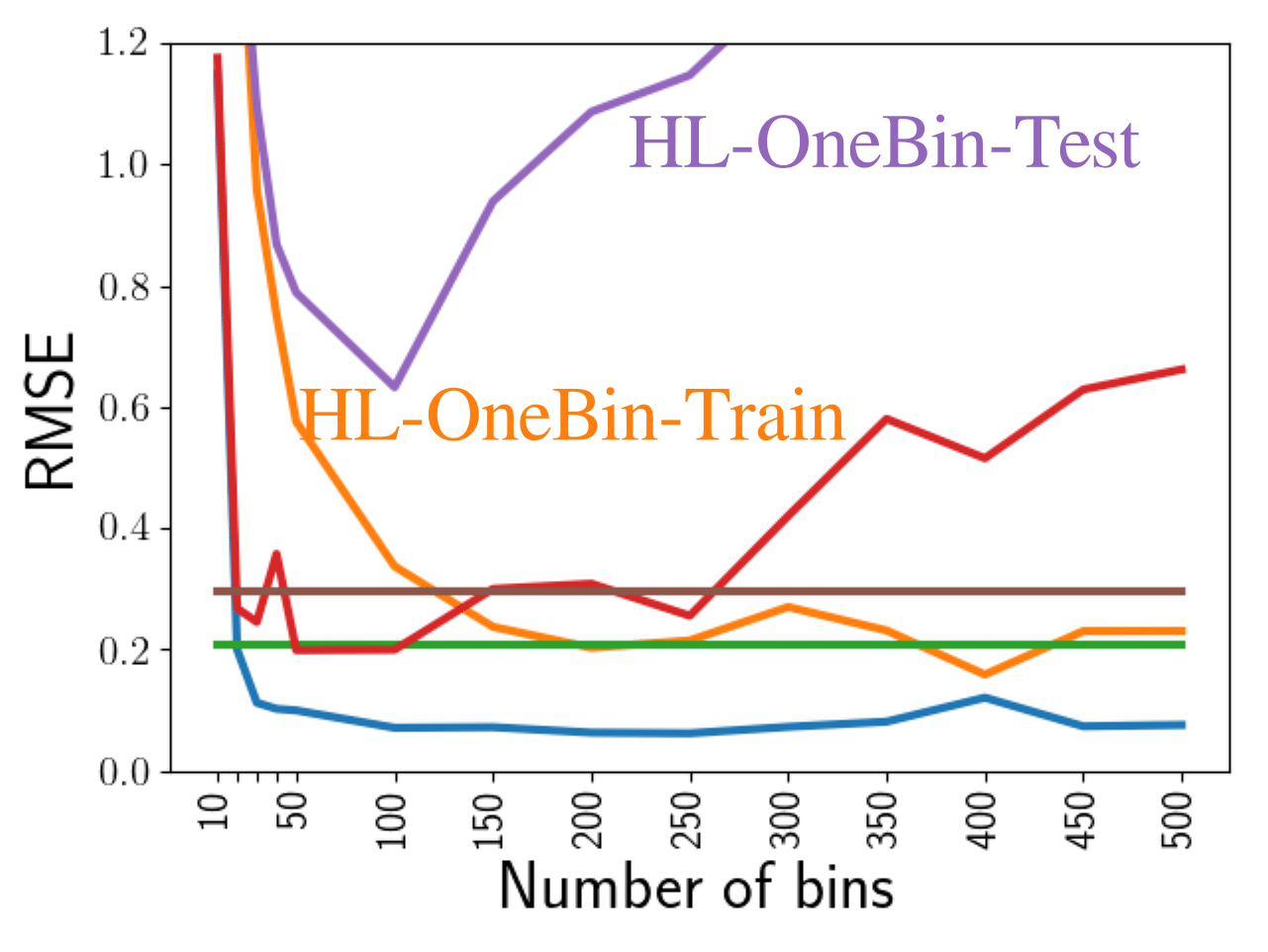}
         \caption{RMSE with changing bins}
         \label{fig:mae-bins}
 \end{subfigure}
 \begin{subfigure}[b]{\figwidthfour}
         \includegraphics[width=\textwidth]{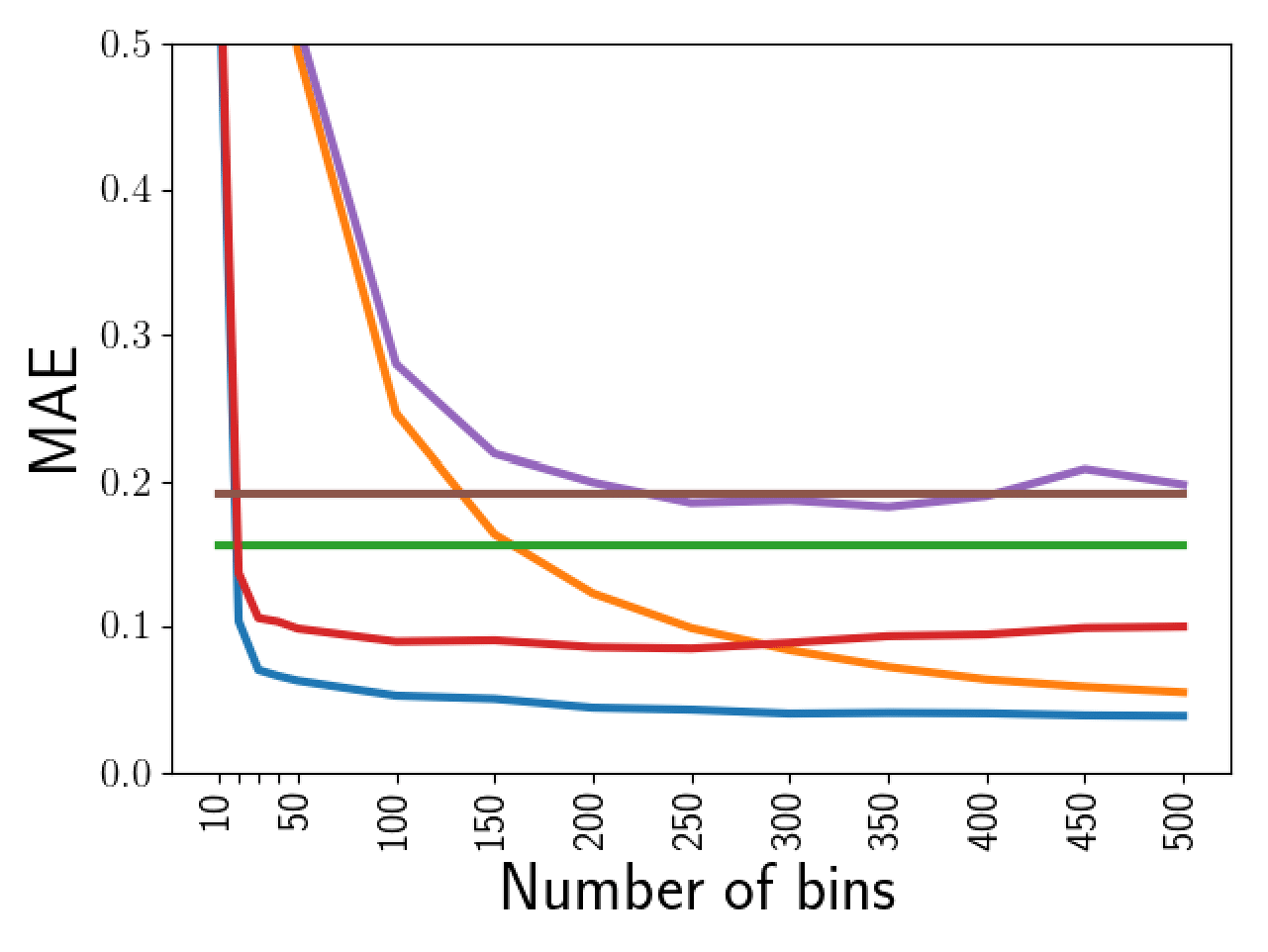}
         \caption{MAE with changing bins}
         \label{fig:rmse-bins}
 \end{subfigure}   
  \begin{subfigure}[b]{\figwidthfour}
         \includegraphics[width=\textwidth]{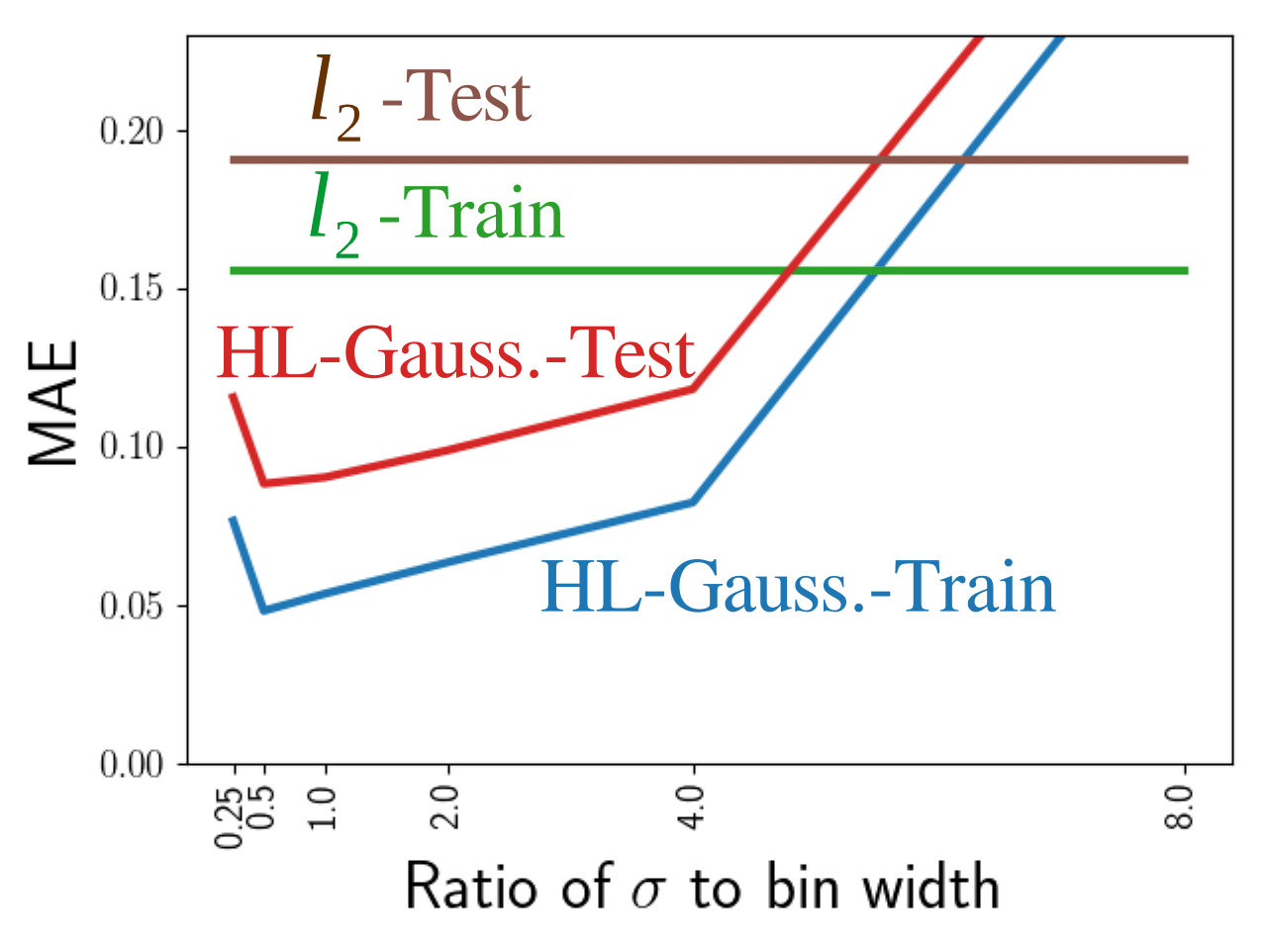}
         \caption{MAE with changing $\sigma$}
         \label{fig:mae-sigma}
 \end{subfigure} 
 \vspace{-0.7cm}
\caption{Investigating the bias-variance trade-offs for HL-Gaussian, on the CT Position dataset. Both training and testing performance are reported, with Regression included as a baseline, which does not vary with bins or $\sigma$. For a wide range of both bins and $\sigma$, HL-Gaussian improves performance over Regression, emphasizing that this performance improvement is not due to carefully setting these additional parameters in the loss. For MAE, for both bins and $\sigma$, the training and testing error have similar shapes, suggesting that there is no significant effect of overfitting. An interesting point is that according to the RMSE, however, there does appear to be some overfitting. Overfitting is typically stated agnostic to the loss, but here MAE and RMSE show different trends. We hypothesize that this result occurs because RMSE is magnifying errors that are slightly larger, even though in terms of absolute error, the performance only slightly degrades.}
\label{fig:biasvariance}
\vspace{-0.3cm}
\end{figure*}

\myparagraph{Overall performance and conclusions (Tables \ref{tab:ctscan}, \ref{tab:yearpred}, \ref{tab:bike_narrow}).}
We first report the relative performance of all these models, on the CT Position dataset (Table \ref{tab:ctscan}) and, in Appendix \ref{app_tables}, the Song Year dataset (Table \ref{tab:yearpred}) and Bike Sharing dataset (Table \ref{tab:bike_narrow}). The overall conclusions are that the HL-Gaussian never harms performance---slightly improving performance on the Song Year dataset---and otherwise can significantly improve performance over alternatives---on both the CT Position and Bike Sharing datasets. We only report the full set of algorithms for CT Position, and more in-depth experiments understanding the result on that domain. 

%
%
%

\myparagraph{Learning other distributions is not effective (Table \ref{tab:ctscan}).}\\
HL-Gaussian improves performance, but the other distribution-learning approaches appear to have little advantages, as shown in Table \ref{tab:ctscan}. 
HL-OneBin and HL-Uniform can actually do worse than Regression. MDN provides only minor gains over Regression. Interestingly, 
it has been shown MDN suffers from numerical instabilities, making training difficult \cite{oord2016pixel,rupprecht2016learning}.

A related idea to learning the distribution explicitly is to use data augmentation, through label smoothing.
We therefore also compared to directly modifying the labels and gradients, with $\ell_2$-Noise and $\ell_2$-Clipping. These models do perform slightly better than Regression for some settings, but do not achieve the same gains as HL-Gaussian. 

\myparagraph{The bias-variance trade-off in the loss definition is not significantly impacting performance (Figure \ref{fig:biasvariance}).}\\
If one fixes the possible range of the output variable, the distribution becomes more and more expressive as the number of bins increases. The model could have a higher chance of overfitting in this situation. Reducing the number of bins, on the other hand, introduces discretization error and increases the bias. 
Further, the entropy parameter $\sigma^2$ introduces a bias-variance trade-off, making the target distribution more uniform as entropy increases---likely resulting in lower variance---but also washing out the signal---incurring high bias. 
The selection of these parameters, therefore, may provide a opportunity to influence this bias-variance trade-off, and improve performance by essentially optimizing the loss for a problem. The ability for the user to select these parameters could explain some of the performance gains in recent results \citep{gao2017deep,bellemare2017distributional}, compared to standard losses that cannot be tuned. 

We tested the impact of varying the number of bins, and the entropy $\sigma^2$ for HL-Gaussian.
%
We found that these parameters, especially the entropy, can have an impact on performance, but that the results were much more robust to changing these parameters than might be expected (reported in more depth in Figure \ref{fig:biasvariance}). It does not seem to be the case, therefore, that the tuning of these hyperparameters is the primary explanation for the improved performance.

\begin{table*}[ht!]
\vspace{-0.2cm}
	\centering
\setlength\tabcolsep{2.pt}

\begin{tabular}{llllll}
	\hline
	& Loss               & Default                               & Fixed                                   & Initialized                            & Random                                   \\
	\hline
	\multirow{2}{*}{Train MAE}  & Regression         & $ 15.624$ {\scriptsize $(\pm 3.353)$} & $ 288.667$ {\scriptsize $(\pm 13.344)$} & $ 24.814$ {\scriptsize $(\pm 5.917)$}  & $ 923.335$ {\scriptsize $(\pm 15.665)$}  \\
	  & HL-Gaussian & $ 5.266$ {\scriptsize $(\pm 0.155)$}  & $ 16.890$ {\scriptsize $(\pm 2.026)$}   & $ 5.971$ {\scriptsize $(\pm 0.103)$}   & $ 247.851$ {\scriptsize $(\pm 9.686)$}   \\
	  \hline
	\multirow{2}{*}{Train RMSE} & Regression         & $ 20.774$ {\scriptsize $(\pm 4.713)$} & $ 399.664$ {\scriptsize $(\pm 17.364)$} & $ 34.185$ {\scriptsize $(\pm 8.868)$}  & $ 1224.689$ {\scriptsize $(\pm 17.199)$} \\
	 & HL-Gaussian & $ 7.097$ {\scriptsize $(\pm 0.169)$}  & $ 24.744$ {\scriptsize $(\pm 2.470)$}   & $ 7.834$ {\scriptsize $(\pm 0.130)$}   & $ 555.212$ {\scriptsize $(\pm 22.567)$}  \\
	 \hline
	  \rowcolor{aliceblue}  & Regression         & $ 19.110$ {\scriptsize $(\pm 3.034)$} & $ 291.980$ {\scriptsize $(\pm 13.587)$} & $ 28.310$ {\scriptsize $(\pm 6.018)$}  & $ 930.481$ {\scriptsize $(\pm 19.094)$}  \\
	\rowcolor{aliceblue}   \multirow{-2}{*}{Test MAE} & HL-Gaussian & $ 8.992$ {\scriptsize $(\pm 0.235)$}  & $ 20.296$ {\scriptsize $(\pm 1.832)$}   & $ 10.089$ {\scriptsize $(\pm 0.087)$}  & $ 260.863$ {\scriptsize $(\pm 10.751)$}  \\
	   \hline
	\multirow{2}{*}{Test RMSE}  & Regression         & $ 29.512$ {\scriptsize $(\pm 3.622)$} & $ 403.070$ {\scriptsize $(\pm 17.237)$} & $ 44.418$ {\scriptsize $(\pm 12.183)$} & $ 1231.502$ {\scriptsize $(\pm 21.574)$} \\
	  & HL-Gaussian & $ 19.980$ {\scriptsize $(\pm 2.169)$} & $ 31.288$ {\scriptsize $(\pm 2.145)$}   & $ 23.161$ {\scriptsize $(\pm 3.682)$}  & $ 589.087$ {\scriptsize $(\pm 22.525)$}  \\
	\hline
\end{tabular}
\vspace{-0.3cm}		
	\caption{Representation experiment results on CT Position dataset. All the numbers are multiplied by $10^2$. We tested (a) swapping the representations and re-learning on the last layer (\textbf{Fixed}), (b) initializing with the other's representation (\textbf{Initialized}), (c) and using the same fixed random representation for both (\textbf{Random}) and only learning the last layer. We highlight the Test MAE, though the other rows have similar trends. Using the HL-Gaussian representation for Regression (first column, Fixed) causes a sudden spike in error, even though the last layer in Regression is re-trained. This suggests the representation is tuned to HL-Gaussian. The representation does not even seem to give a boost in performance, as an initialization (second column, Initialization). Finally, even with the same random representation, where HL-Gaussian cannot be said to improve the representation, HL-Gaussian still obtains significantly better performance, solely from optimizing the last layer with a different loss than Regression. }
	\label{tab:representation}
\vspace{-0.3cm}	
\end{table*}

\begin{figure*}[ht!]
\centering
 \begin{subfigure}[b]{\figwidthfour}
         \includegraphics[width=\textwidth]{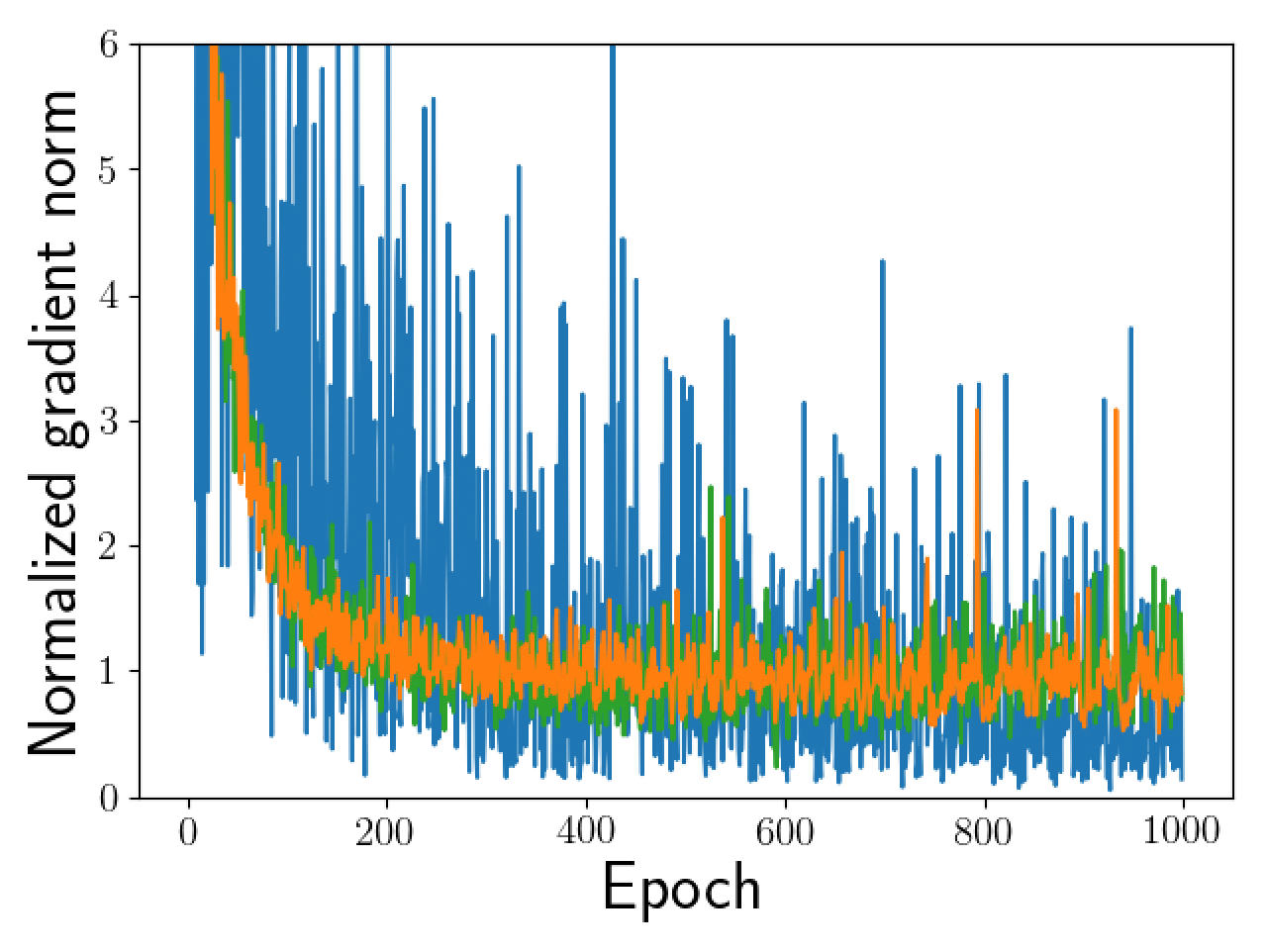}
 \vspace{-0.7cm}	
        \caption{Normalized gradient norm during training}
         \label{fig:mae-bins}
 \end{subfigure}
  \begin{subfigure}[b]{\figwidthfour}
         \includegraphics[width=\textwidth]{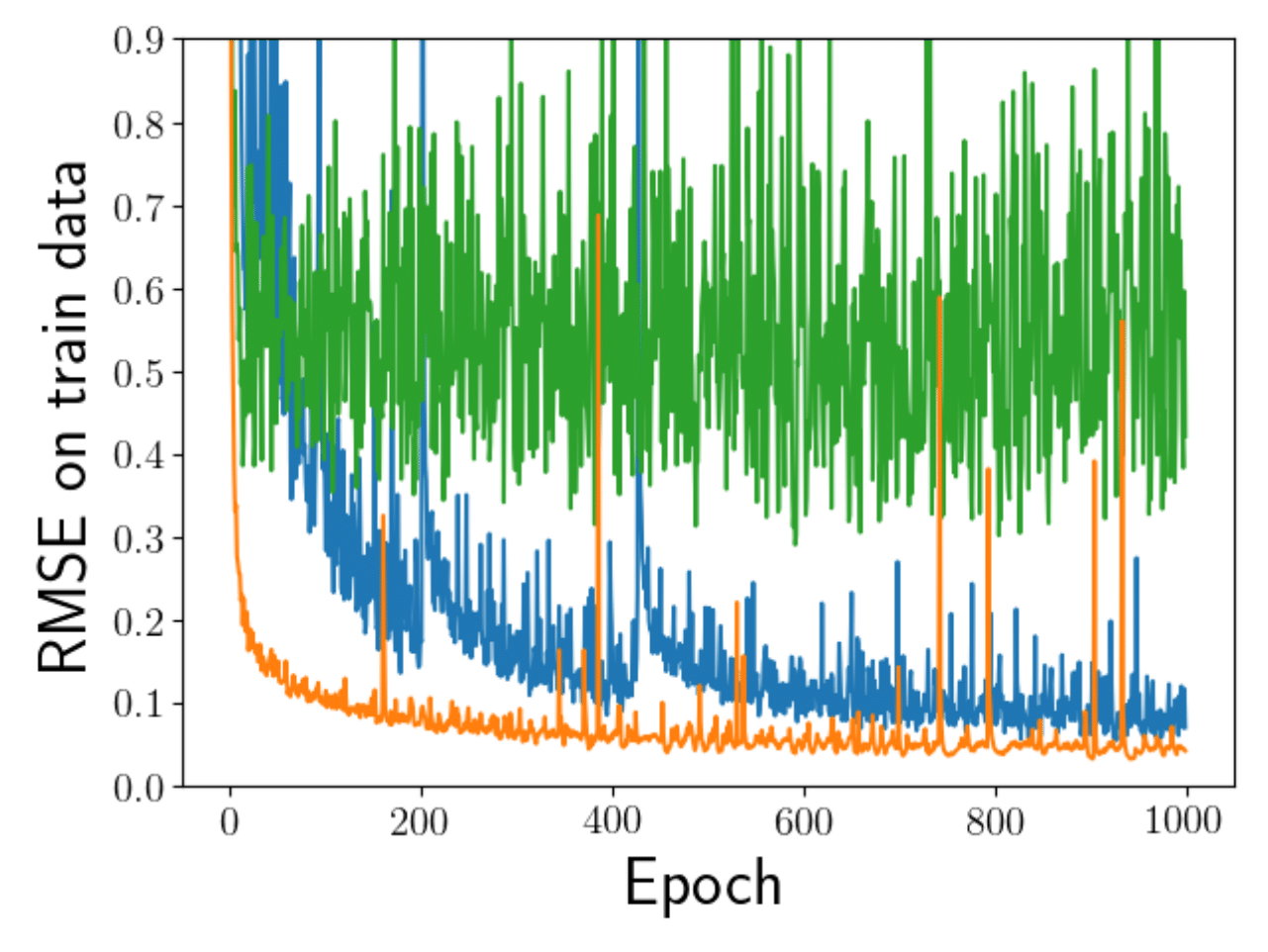}
\vspace{-0.7cm}	
         \caption{RMSE on training data, vs. epochs}
         \label{fig:mae-bins}
 \end{subfigure}
 \begin{subfigure}[b]{\figwidthfour}
         \includegraphics[width=\textwidth]{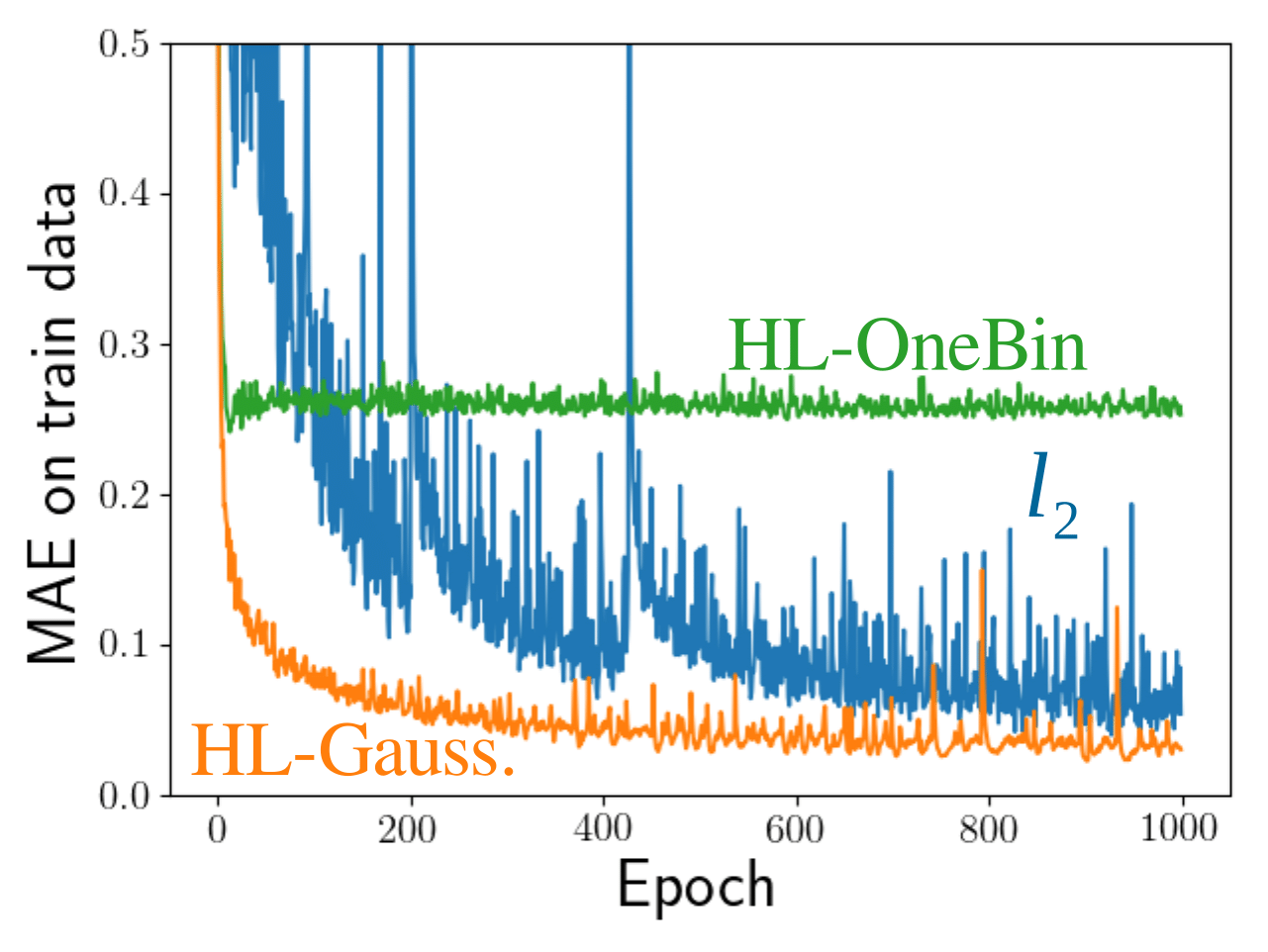}
\vspace{-0.7cm}	
         \caption{MAE on training data, vs. epochs}
         \label{fig:rmse-bins}
 \end{subfigure}   
\vspace{-0.7cm}
\caption{The norm of the gradient and error values for the training on the CT Position dataset, for three training objectives. The behaviour for the testing error is similar to the training error, and so is included in Appendix \ref{app_results}. The gradient norms are normalized by the median value, where the median norms are 0.0014 for Regression, 0.0640 for HL-OneBin, and 0.0305 for HL-Gaussian. The median norm for HL-OneBin is an order of magnitude larger than HL-Gaussian, so though it is not variable, it is consistently larger. Because targets are normalized, the median norm for $\ell_2$ is actually lower, but it has significantly more variability.}
\label{fig:lossplots}
\vspace{-0.3cm}
\end{figure*}

\myparagraph{The learned representation is not better (Table \ref{tab:representation}).}\\
Learning a distribution, as opposed to a single statistic, provides a more difficult target---one that could require a better representation.
The hypothesis is that amongst the functions $f$ in your function class $\mathcal{F}$, there is a set of functions
that can predict the targets almost equally well. To distinguish amongst these functions, a wider range of tasks can make it more likely to select the true function, or at least one that generalizes better. 


We conducted three experiments to test the hypothesis than an improved representation is learned. We first trained with HL-Gaussian and $\ell_2$, to obtain their representations. We tested (a) swapping the representations and re-learning only the last layer, (b) initializing with the other's representation, (c) and using the same fixed random representation for both. For (a) and (c), the optimizations for both are convex, since the representation is fixed.
The results in Table \ref{tab:representation}, are surprisingly conclusive: using the representation from HL-Gaussian does not improve performance of $\ell_2$, and even under a random representation, HL-Gaussian performs significantly better than $\ell_2$. 
This suggests that HL-Gaussian is not causing a more useful or more general representation to be learned, as otherwise $\ell_2$ should be able to take advantage of that representation. 

%
%

\myparagraph{The softmax nonlinearity is not the main cause (Table \ref{tab:ctscan}).}
The HL-Gaussian can be seen as a generalized linear model, where a small amount of non-linearity is introduced from the transfer.
The level of nonlinearity is similar to that in the cross-entropy loss, and the effect should be small because each transformed output $\wvec_i^\top \phivec(\xvec)$ has to predict a probability value. This contrasts with an alternative way to use a softmax layer---which we call $\ell_2$+Softmax---which gets to tune the softmax layer to directly predict $y$ given $\xvec$. Such a layer has additional parameters to predict one target (100 additional parameters, for 100 bins). This contrast the HL-Gaussian, which has also 100 bins but has to predict 100 targets instead of just one target.\footnote{It is possible that having 100 extra parameters in the last layer makes it possible to benefit from randomness, over the $\ell_2$. We ran experiments enabling the $\ell_2$ to have 100 outputs, each predicting the target but with different initial weights. Even selecting the best of the 100 outputs \emph{on the test data} only slightly improved performance, with a test MAE of 18.421.} 

Despite the differences between the role of the softmax in HL-Gaussian and $\ell_2$+Softmax, we provide this comparison to provide some insight into potential nonlinearities introduced by the loss. The result in Table \ref{tab:ctscan} shows that this softmax layer can improve performance (to 12.720), but not as significant as HL-Gaussian (8.992). This is particularly intriguing, because as mentioned above, $\ell_2$+Softmax can much more flexibly tune the nonlinear softmax layer. The ability to outperform $\ell_2$+softmax-layer emphasizes that there are properties of the HL causing improvements beyond the use of the softmax.

\myparagraph{HL-Gaussian trains fast (Figure \ref{fig:lossplots}).}\\
We trained $\ell_2$, HL-OneBin, and HL-Gaussian on the CT Position dataset with no dropout to find the role of the loss function on the rate of convergence. We also computed the norm the gradient w.r.t. the parameters of the last layer after each epoch, and normalized the gradient norms of each model by their median to compare their variability.  As shown in Figure \ref{fig:lossplots}, HL-Gaussian has significantly better behaved gradients, than $\ell_2$. Correspondingly, it converges significantly faster and more smoothly. The other two methods that more carefully controlled gradients---$\ell_2$-Noise and $\ell_2$-Clip---provided the next best gains to HL-Gaussian.

\section{Conclusion}

We introduced a novel loss for regression, called the Histogram Loss (HL), that explicitly constructs a distribution over targets to predict, rather than directly estimating the mean of the target conditioned on inputs.
The loss involves minimizing the KL-divergence between a predicted distribution and this target distribution.
To make this loss efficient to compute, without significantly reducing modeling power, we restrict the class of approximation densities to histogram densities.
We highlight that for a particular setting of the HL---with a target Gaussian distribution---the norm of the gradient does not grow large or vary widely. Combined with recent results that show reducing training steps for stochastic gradient results in improved generalization provide some theoretical justification for why we observe such strong performance of HL-Gaussian in practice. 
We conduct a series of experiments to identify this gain, with evidence that the main role is not due to overfitting or an improved representation, but rather due to the fact that the HL can be optimized in a smaller number of steps, with smoother gradients. 

The introduction of the HL provides several avenues to improve our choice of loss function. One direction is to more explicitly take advantage of the specification of the target distribution. 
In this work, we considered this loss only for a fixed set of bins, widths and variance parameter $\sigma$ for the target distribution. To be more agnostic to these choices, we demonstrated performance across possible parameter settings. However, these parameters could be determined using meta-parameter optimization strategies, such as cross validation, or even learning strategies with particular objectives for these parameters. 
The key property to make the HL easy to specify and optimize was the use of a histogram to predict the target; the derivation does not prevent also optimizing the bins centers, widths and variances. 

Overall, this work provides some unification of recent results using soft targets, through the introduction of the HL. We hope for it to facilitate  discussion and development on the design of losses that promote learning, and direct further investigation into the importance of the optimization properties of these losses. 

\section*{Acknowledgments}
We would like to thank Alberta Innovates for funding AMII (the Alberta Machine Intelligence Institute) and this research. 

\bibliographystyle{icml2018}
\bibliography{sample.bib}

\newpage
\clearpage

\appendix
 

\section{Explicit gradient computations} \label{app_calculations}

Let $b_i = \phivec_\thetavec(\xvec)^\top \wvec_i$ and $e_i = \exp(b_i)$. Then, since $f_j(\xvec) =  \frac{e_j}{\sum_{l=1}^\nbins e_l}$, for $j \neq i$
\begin{align*}
\frac{\partial}{\partial b_i} f_j(\xvec) 
&= \frac{\partial}{\partial b_i} \frac{e_j}{\sum_{l=1}^\nbins e_l} 
= -\frac{e_j}{\left(\sum_{l=1}^\nbins e_l\right)^2} e_i   \\
&= - f_j(\xvec) f_i(\xvec) 
\end{align*}
For $j = i$, we get
\begin{align*}
\frac{\partial}{\partial b_i} f_j(\xvec) 
&= \frac{e_i}{\sum_{l=1}^\nbins e_l} - \frac{e_i}{\left(\sum_{l=1}^\nbins e_l\right)^2} e_i   \\
&= f_i(\xvec)[1 - f_i(\xvec)]  
\end{align*}
Consider now the gradient of the HL, w.r.t $b_i$
\begin{align*}
\frac{\partial}{\partial b_i} \sum_{j=1}^\nbins \distweight_j \log f_j(\xvec) 
&= \sum_{j=1}^\nbins \distweight_j \frac{1}{f_j(\xvec)} f_j(\xvec) (1_{i = j} - f_i(\xvec))  \\
&= \sum_{j=1}^\nbins \distweight_j (1_{i = j} - f_i(\xvec) )\\ 
&=  \distweight_i - f_i(\xvec) \sum_{i=1}^\nbins \distweight_i \\ 
&=  \distweight_i - f_i(\xvec)
\end{align*}
Then
\begin{align*}
\frac{\partial}{\partial \wvec_i} \sum_{j=1}^\nbins \distweight_i \log f_i(\xvec) 
&=  \left(\distweight_i - f_i(\xvec) \right) \phivec_\thetavec(\xvec)\\
\frac{\partial}{\partial \thetavec} \sum_{j=1}^\nbins \distweight_i \log f_i(\xvec) 
&=  \sum_{i=1}^\nbins  \left(\distweight_i - f_i(\xvec) \right) \nabla \wvec_i^\top \phivec_\thetavec(\xvec)
\end{align*}
where $J \phivec_\thetavec(\xvec)$ is the Jacobian of $\phivec_\thetavec$.

\begin{table}[h]
\begin{center}
\begin{tabular}{|c c c c c|}
 \hline
 Dataset & \# train & \# test & \# feats & $Y$ range \\ [0.5ex] 
 \hline\hline
 Song Year & 463715 & 51630 & 90 & [1922,2011] \\ 
 \hline
 CT Position & 42800 & 10700 & 385 & [0,100] \\ [0ex] 
 \hline
 Bike Sharing & 13911 & 3478 & 16 & [0,1000] \\ [0ex] 
 \hline
\end{tabular}
\caption{Overview of the datasets used in the experiments.}\label{tab:Dataset Table}
\end{center}
\vspace{-0.5cm}
\end{table}

\section{Dataset details.}\label{app_datasets}

We include an overview of the datasets in Table \ref{tab:Dataset Table}. We additionally show a histogram of their targets, in Figure \ref{fig:histograms}. 

\section{Additional experiments}\label{app_tables} 

We provide overall performance results for the two other datasets. We include the learning curves on test data for CT Position, corresponding to Figure \ref{fig:lossplots} in the main text. 

\subsection{Test Learning Curves for CT Position dataset}\label{app_results}

We include additional graphs for the variability in the RMSE and MAE, for the three objectives HL-Gaussian, HL-OneBin and $\ell_2$, for test data in Figure \ref{fig:lossplots}.

\begin{figure}
    \centering
    \begin{minipage}{0.4\textwidth}
        \centering
        \includegraphics[width=0.9\textwidth]{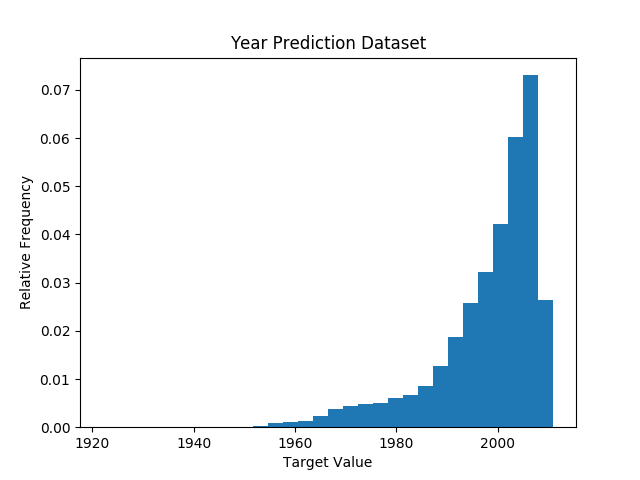}
    \end{minipage}\hfill
    \begin{minipage}{0.4\textwidth}
        \centering
        \includegraphics[width=0.9\textwidth]{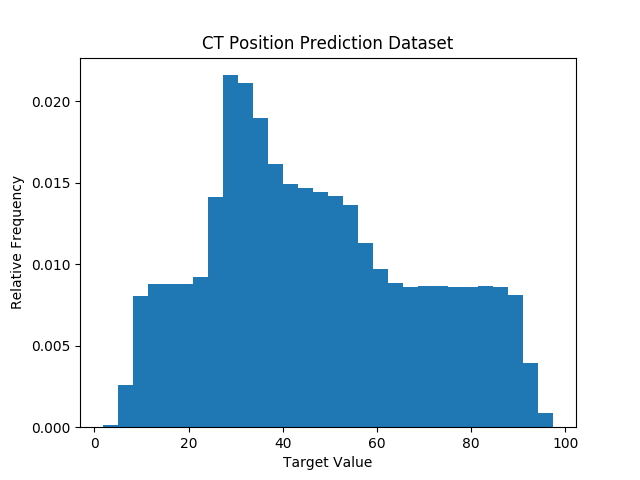}
    \end{minipage}
    \begin{minipage}{0.4\textwidth}
       	\centering
       	\includegraphics[width=0.9\textwidth]{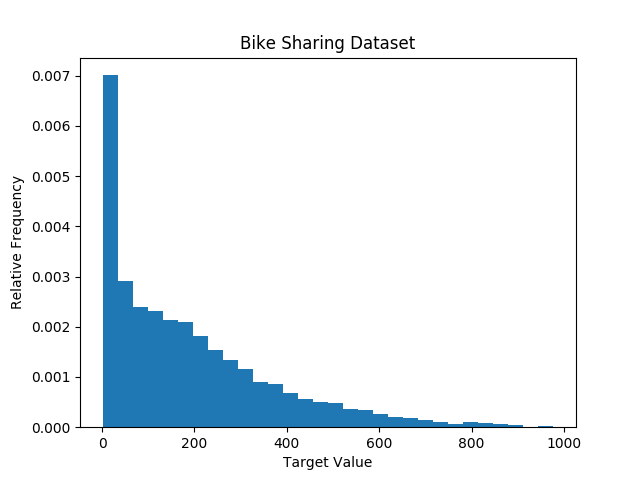}
    \end{minipage}
    \caption{Histogram of the target values for the three datasets. }
    \label{fig:histograms}
\end{figure}

\begin{figure}
\centering
 \begin{subfigure}[b]{\figwidthfour}
         \includegraphics[width=\textwidth]{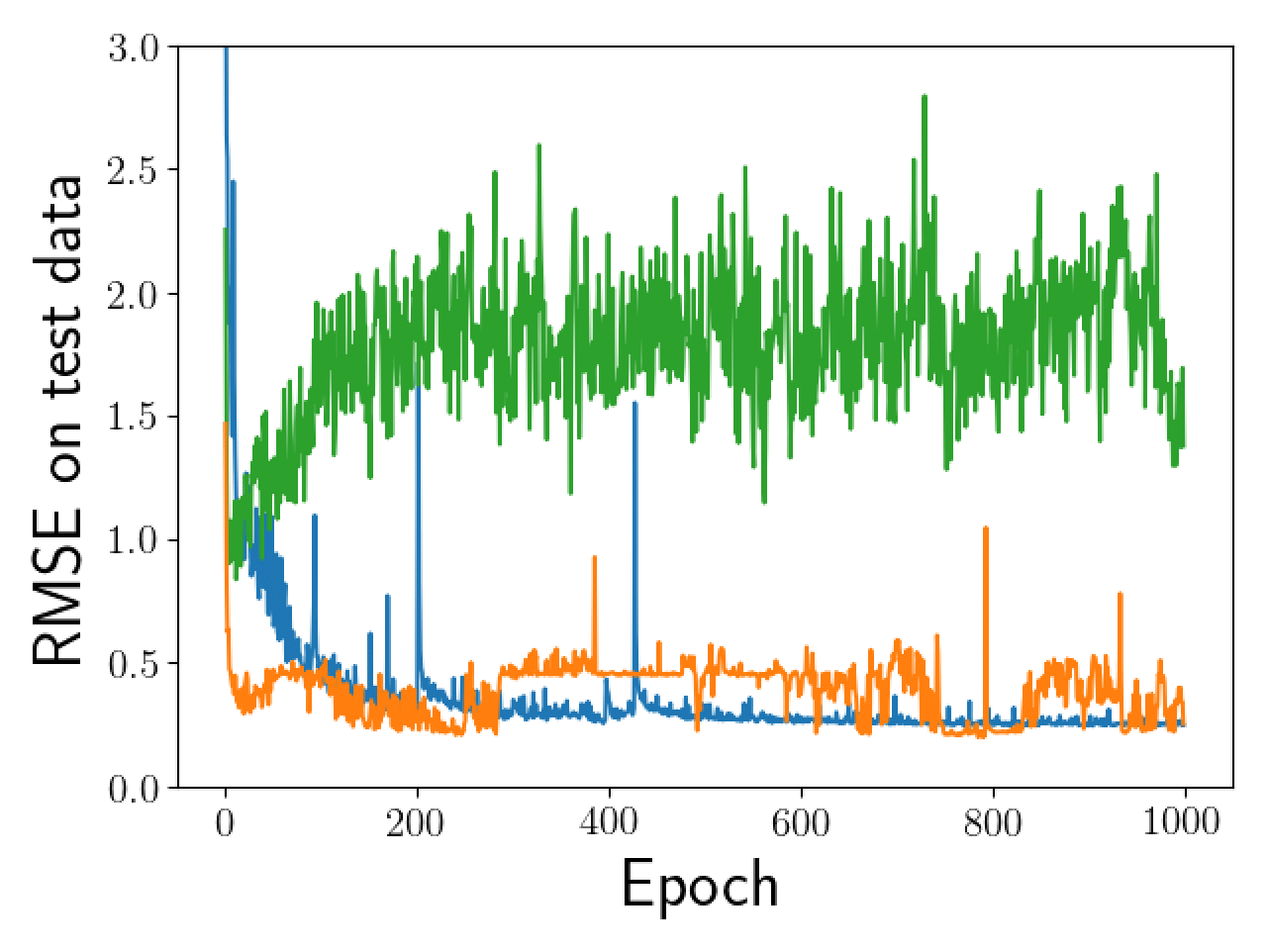}
         \caption{RMSE on testing data, vs. epochs}
         \label{fig:rmse-sigma}
 \end{subfigure}   
  \begin{subfigure}[b]{\figwidthfour}
         \includegraphics[width=\textwidth]{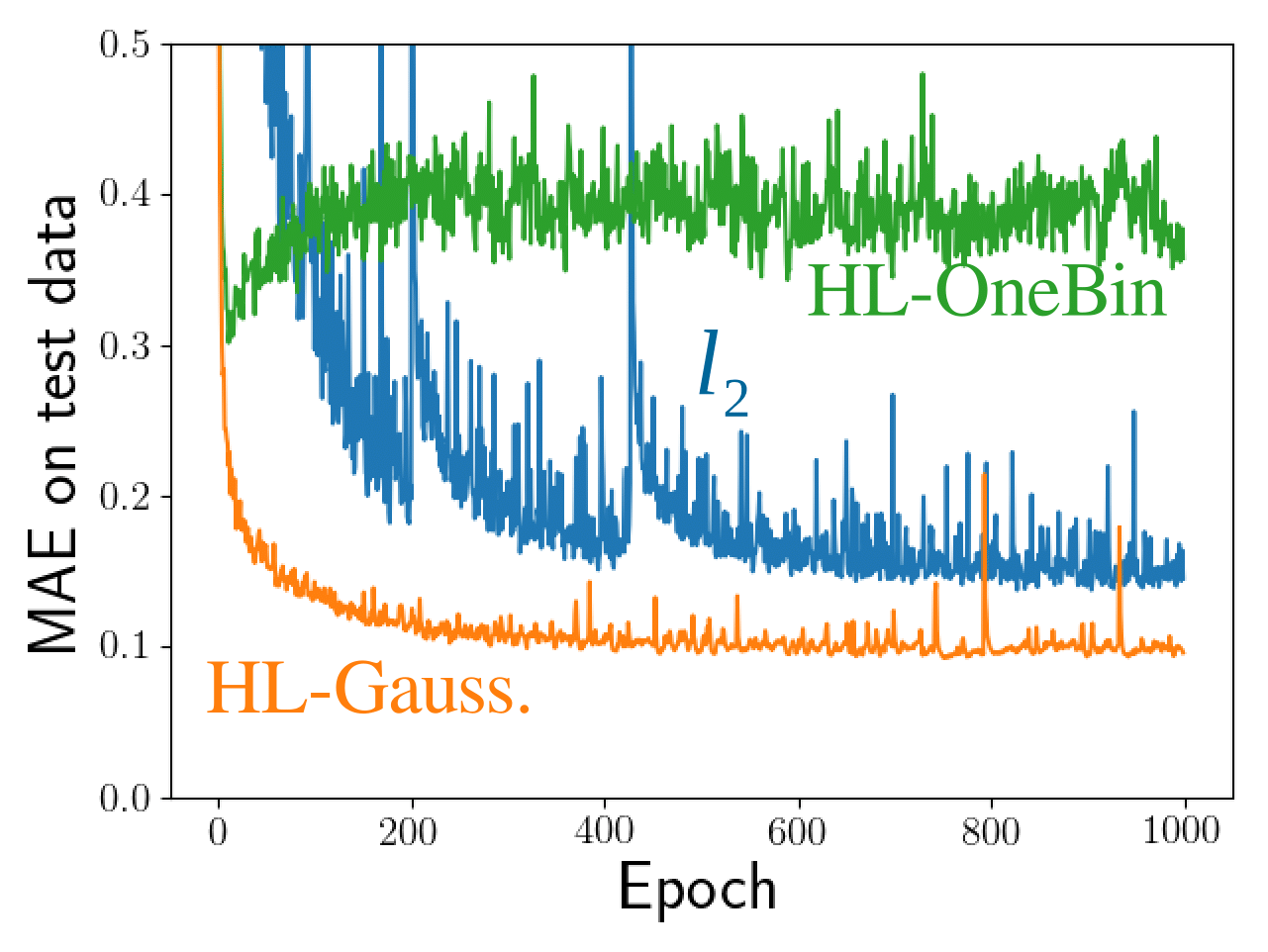}
         \caption{MAE on testing data, vs. epochs}
         \label{fig:mae-sigma}
 \end{subfigure} 
\caption{The error values for the testing data on the CT Position dataset, for three training objectives. }
\label{fig:lossplots}
\end{figure}

\begin{table*}
	\centering
	\setlength\tabcolsep{.7pt}
	\begin{tabular}{lllllll}
		\hline
		Method             & Train objective   & Train MAE   & Train RMSE   & Test objective   & Test MAE   & Test RMSE   \\
		\hline
		Linear Reg.  & $ 9114.456$ {\scriptsize $(\pm 6.524)$} & $ 679.285$ {\scriptsize $(\pm 0.264)$} & $ 954.696$ {\scriptsize $(\pm 0.342)$} & $ 9129.131$ {\scriptsize $(\pm 26.215)$} & $ 679.646$ {\scriptsize $(\pm 0.741)$} & $ 955.461$ {\scriptsize $(\pm 1.370)$}  \\
		$\ell_2$         & $ 0.900$ {\scriptsize $(\pm 0.001)$}    & $ 575.997$ {\scriptsize $(\pm 2.171)$} & $ 813.868$ {\scriptsize $(\pm 0.825)$} & $ 0.955$ {\scriptsize $(\pm 0.004)$}     & $ 602.393$ {\scriptsize $(\pm 2.026)$} & $ 869.569$ {\scriptsize $(\pm 1.923)$}  \\
		HL-Gaussian & $ 320.475$ {\scriptsize $(\pm 0.039)$}  & $ 580.305$ {\scriptsize $(\pm 0.723)$} & $ 846.072$ {\scriptsize $(\pm 0.282)$} & $ 320.260$ {\scriptsize $(\pm 0.052)$}   & $ 591.304$ {\scriptsize $(\pm 1.413)$} & $ 862.656$ {\scriptsize $(\pm 1.683)$}  \\
		\hline
		HL-OneBin   & $ 304.411$ {\scriptsize $(\pm 0.063)$}  & $ 581.230$ {\scriptsize $(\pm 0.967)$} & $ 848.745$ {\scriptsize $(\pm 0.284)$} & $ 304.827$ {\scriptsize $(\pm 0.091)$}   & $ 590.823$ {\scriptsize $(\pm 1.589)$} & $ 863.475$ {\scriptsize $(\pm 1.621)$}  \\
		$\ell_2$+Softmax    & $ 0.895$ {\scriptsize $(\pm 0.003)$}    & $ 569.797$ {\scriptsize $(\pm 3.613)$} & $ 809.113$ {\scriptsize $(\pm 1.679)$} & $ 0.962$ {\scriptsize $(\pm 0.003)$}     & $ 600.607$ {\scriptsize $(\pm 3.064)$} & $ 872.765$ {\scriptsize $(\pm 1.477)$} \\
		\hline
	\end{tabular}
	\vspace{-0.2cm}
	\caption{Performance on the Song Year dataset.  All the numbers are multiplied by $10^2$.}
	\label{tab:yearpred}
	\vspace{-0.2cm}
\end{table*}

\begin{table*}[t]
  \centering
\begin{tabular}{lllllll}
	\hline
	Method             & Train objective   & Train MAE   & Train RMSE   & Test objective   & Test MAE   & Test RMSE   \\
	\hline
	Linear Reg.  & $ 9125.643$               & $ 679.557$     & $ 955.282$      & $ 9044.316$           & $ 680.050$ & $ 951.016$  \\
	$\ell_2$         & $ 0.904$        & $ 574.110$     & $ 817.002$      & $ 0.997$    & $ 610.344$ & $ 888.808$  \\
	HL-Gaussian & $ 320.334$      & $ 576.914$     & $ 844.565$      & $ 322.028$  & $ 598.490$ & $ 875.857$  \\
	\hline
	HL-OneBin   & $ 304.389$      & $ 581.066$     & $ 848.344$      & $ 307.988$  & $ 601.927$ & $ 879.877$  \\
	\hline
\end{tabular}
  \vspace{-0.2cm}
  \caption{Performance on the Song Year dataset with the authors' suggested train and test splits.  All the numbers are multiplied by $10^2$.}
  \label{tab:yearpred_fixed}
  \vspace{-0.2cm}
\end{table*}

\begin{table*}[ht!]
	\centering
    \setlength\tabcolsep{.7pt}
	\begin{tabular}{lllllll}
		\hline
		Method             & Train objective   & Train MAE   & Train RMSE   & Test objective   & Test MAE   & Test RMSE   \\
		\hline
		 $\ell_2$        & $ 0.04$ {\scriptsize $(\pm 0.00)$}          & $ 1343.51$ {\scriptsize $(\pm 41.46)$}  & $ 1865.55$ {\scriptsize $(\pm 60.99)$}  & $ 0.22$ {\scriptsize $(\pm 0.00)$}           & $ 2899.84$ {\scriptsize $(\pm 52.44)$}  & $ 4601.21$ {\scriptsize $(\pm 80.54)$}  \\
		 HL-Gaussian       & $ 212.43$ {\scriptsize $(\pm 1.12)$}        & $ 1667.68$ {\scriptsize $(\pm 30.12)$}  & $ 2645.08$ {\scriptsize $(\pm 41.10)$}  & $ 254.60$ {\scriptsize $(\pm 1.01)$}         & $ 2495.21$ {\scriptsize $(\pm 13.28)$}  & $ 4182.06$ {\scriptsize $(\pm 57.35)$}  \\
		 \hline
		 HL-OneBin  & $ 169.62$ {\scriptsize $(\pm 3.53)$}        & $ 1768.68$ {\scriptsize $(\pm 39.17)$}  & $ 2871.14$ {\scriptsize $(\pm 56.85)$}  & $ 287.13$ {\scriptsize $(\pm 5.80)$}         & $ 2607.38$ {\scriptsize $(\pm 15.94)$}  & $ 4373.20$ {\scriptsize $(\pm 53.96)$}  \\
		 $\ell_2$+Softmax   & $ 0.07$ {\scriptsize $(\pm 0.01)$}          & $ 1750.31$ {\scriptsize $(\pm 127.58)$} & $ 2463.93$ {\scriptsize $(\pm 150.76)$} & $ 0.22$ {\scriptsize $(\pm 0.00)$}           & $ 2886.93$ {\scriptsize $(\pm 70.80)$}  & $ 4544.68$ {\scriptsize $(\pm 41.25)$}  \\
		 \hline
	\end{tabular}
	\vspace{-0.2cm}
	\caption{Performance on the Bike Sharing dataset. All the numbers are multiplied by $10^2$. }
	\label{tab:bike_narrow}
	\vspace{-0.2cm}
\end{table*}

\subsection{Experiments on Song Year dataset}

For the Song Year dataset, we include results both for random train-test splits and report results for the fixed train/test split recommended by the authors of the dataset to avoid the effect of an artist having songs in both the train and test sets. 

For this dataset, both HL-Gaussian and HL-OneBin outperform $\ell_2$ only slightly, and perform similarly to each other. The $\ell_2$ loss with a nonlinear softmax layer also performs about the same, suggesting the main (small) gain for this dataset is from this nonlinearity. This further suggests that the $\ell_2$ is likely a suitable loss for this problem, and there is little to gain for switching to the HL. There is a slightly larger gain for HL-Gaussian in Table \ref{tab:yearpred_fixed} for the training/test split suggested by the authors of this data, but still not nearly as large as CT Position or Bike Sharing. 

\subsection{Experiments on Bike Sharing dataset}

We provide a comparison of performance on the Bike Sharing dataset in Table \ref{tab:bike_narrow}. We used early stopping to avoid overfitting, because on this dataset, dropout was ineffective.

The network for Bike Sharing uses four hidden layers of width 64, but we additionally tested a network architecture with four hidden layers of width 512. For this wider network, $\ell_2$ was able to get better final TEST MAE and Test RMSE performance of  $2402.67$  and $4014.31$ respectively. However, performance was quite a bit more variable during learning---likely due to the overparameterization. Future work is to better understand the effect of different network architectures on the performance of the different losses.

\end{document}